\documentclass{article}

\usepackage{colm2024_conference}
\usepackage{times}
\usepackage{graphicx}
\usepackage{amsmath,amsfonts,amssymb}
\usepackage{amsthm}
\usepackage{algorithm}
\usepackage{algorithmic}
\usepackage{booktabs}
\usepackage{multirow}
\usepackage{hyperref}
\usepackage{natbib}
\usepackage{pifont}
\usepackage[table]{xcolor}
\usepackage{wrapfig}

\colmfinalcopy

\newcommand{\Qtot}{Q_{\text{tot}}}
\newcommand{\Qi}{Q_{i}}
\newcommand{\Ebar}{\bar{R}}

\usepackage{tcolorbox}
\definecolor{darkgray}{rgb}{0.5,0.5,0.5}
\definecolor{lightpurple}{rgb}{0.9, 0.9, 1.0}
\definecolor{ForestGreen}{HTML}{228B22}
\definecolor{revblue}{rgb}{0.0, 0.0, 0.8}

\newtheorem{definition}{Definition}[section]

\newtheorem{theorem}{Theorem}[section]
\newtheorem{corollary}{Corollary}[section]
\newtheorem{assumption}{Assumption}[section]

\title{Agent Q-Mix: Selecting the Right Action for LLM Multi-Agent Systems through Reinforcement Learning}

\author{
 \textbf{Eric Hanchen Jiang\textsuperscript{1*}},
  \textbf{Levina Li\textsuperscript{1*}},
  \textbf{Rui Sun\textsuperscript{1}},
  \textbf{Xiao Liang\textsuperscript{1}},
  \textbf{Yubei Li\textsuperscript{1}},
  \\
  \textbf{ Yuchen Wu\textsuperscript{2}},
  \textbf{Haozheng Luo\textsuperscript{3}},
  \textbf{Hengli Li},
  \textbf{Zhi Zhang\textsuperscript{1}},
  \textbf{Zhaolu Kang},
    \\
  \textbf{ Kai-Wei Chang\textsuperscript{1}},
  \textbf{Ying Nian Wu\textsuperscript{1}}
\\
\\
  \textsuperscript{1}University of California Los Angeles,
  \textsuperscript{2}University of Washington, \\
  \textsuperscript{3}Northwestern University
  }

\begin{document}
\maketitle

\begingroup
\renewcommand\thefootnote{}
\footnotetext{
\textsuperscript{*}Equal contribution. }
\endgroup

\begin{abstract}
Large Language Models (LLMs) have shown remarkable performance in completing various tasks. However, solving complex problems often requires the coordination of multiple agents, raising a fundamental question: how to effectively select and interconnect these agents. In this paper, we propose \textbf{Agent Q-Mix}, a reinforcement learning framework that reformulates topology selection as a cooperative Multi-Agent Reinforcement Learning (MARL) problem. Our method learns decentralized communication decisions using QMIX value factorization, where each agent selects from a set of communication actions that jointly induce a round-wise communication graph. At its core, Agent Q-Mix combines a topology-aware GNN encoder, GRU memory, and per-agent Q-heads under a Centralized Training with Decentralized Execution (CTDE) paradigm. The framework optimizes a reward function that balances task accuracy with token cost. Across seven core benchmarks in coding, reasoning, and mathematics, Agent Q-Mix achieves the highest average accuracy compared to existing methods while demonstrating superior token efficiency and robustness against agent failure. Notably, on the challenging Humanity's Last Exam (HLE) using Gemini-3.1-Flash-Lite as a backbone, Agent Q-Mix achieves 20.8\% accuracy, outperforming Microsoft Agent Framework (19.2\%) and LangGraph (19.2\%), followed by AutoGen and Lobster by OpenClaw. These results underscore the effectiveness of learned, decentralized topology optimization in pushing the boundaries of multi-agent reasoning. Code is available at \url{https://github.com/ericjiang18/Agent-Q-Mix}.
\end{abstract}

\section{Introduction}
\label{sec:intro}

Large language model (LLM)-based multi-agent systems (MAS) have become an increasingly important paradigm for solving tasks that exceed the capability of a single model.
Recent work has shown strong gains in mathematical reasoning~\citep{lei2024macm}, code generation~\citep{islam2024mapcoder}, software development~\citep{he2025llmmultiagent, hong2023metagpt}, and scientific discovery~\citep{ghareeb2025robin}.
At the same time, industrial frameworks such as Microsoft Agent Framework~\citep{microsoft2025agentframework} and LangGraph~\citep{langgraph2024} have accelerated the practical deployment of agentic workflows.
A growing line of research~\citep{zhuge2024gptswarm, zhang2024gdesigner, li2025assemblecrew, zhou2025multiagent} suggests that the \emph{communication topology}, the graph specifying which agents exchange information, is a key determinant of overall system behavior.
In particular, topology affects at least three important properties: \textbf{task performance}, \textbf{token efficiency}, and \textbf{robustness}.
Designing topologies that balance all three remains challenging.

Existing approaches can be broadly divided into two categories.
\textbf{Static methods}~\citep{wu2023autogen, du2023debate, hong2023metagpt, langgraph2024} use fixed communication structures such as chains, stars, or fully connected graphs.
While simple, these patterns do not adapt to task difficulty: easy problems may waste tokens through unnecessary communication, while difficult problems may suffer from insufficient collaboration.
\textbf{Adaptive methods} attempt to learn task-specific topologies, but current approaches still face two important limitations.
First, methods such as G-Designer~\citep{zhang2024gdesigner}, TopoDIM~\citep{sun2025topodim}, and GTD~\citep{jiang2025gtd} rely on a \emph{centralized topology generator} that produces the entire graph at once.
As a result, individual agents do not make independent communication decisions during execution.
This creates a centralized bottleneck and limits flexibility when agents must adapt their communication behavior online.
Second, existing methods generally lack an explicit factorization of the joint topology objective into per-agent contributions, making decentralized action selection difficult to justify.

To address these issues, we propose \textbf{Agent Q-Mix} (Figure~\ref{fig:main_fig}), which reformulates communication topology learning as a cooperative multi-agent reinforcement learning (MARL) problem and solves it using QMIX~\citep{rashid2018qmix}.
Our key idea is to model each agent's communication choice, for example, whether to broadcast, query another agent, debate, verify with a tool, or work independently, as a local discrete action.
The full communication graph is then constructed from the combination of all agents' actions.
This view makes topology learning naturally compatible with \emph{centralized training with decentralized execution} (CTDE).
To implement this, each agent encodes the current communication graph through a topology-aware graph neural network (GNN), maintains temporal memory via a gated recurrent unit (GRU) over communication rounds, and computes per-action Q-values through a multi-layer perceptron (MLP). These individual Q-values are combined by a QMIX monotonic mixing network that ensures consistent decentralized execution.

QMIX is a natural fit for this setting for three reasons.
First, its monotonic value factorization enforces the Individual-Global-Max (IGM) property~\citep{rashid2018qmix}, allowing decentralized greedy action selection with respect to the learned joint value.
Second, the state-conditioned mixing network can model nonlinear interactions between agents while still decomposing the decision process into per-agent utilities.
Third, by combining QMIX with a topology-aware GNN encoder, each agent can condition its decision on the current communication graph and adapt its behavior across multiple communication rounds.

\paragraph{Contributions.}
Our main contributions are as follows:

\begin{itemize}
    \item[\ding{182}] \textbf{Formulation.} We formulate communication-topology learning for LLM-based multi-agent systems as a Networked Multi-Agent Markov Decision Process (Networked MMDP; Section~\ref{sec:mmdp}), and introduce a discrete and interpretable communication action space with six action types that directly induce graph structure.
    \item[\ding{183}] \textbf{Algorithm.} We design a topology-aware agent Q-network based on a GNN-GRU-MLP architecture, together with a QMIX monotonic mixing network that supports CTDE for decentralized communication-action selection.
    \item[\ding{184}] \textbf{Evaluation.} We evaluate Agent Q-Mix on seven core benchmarks spanning mathematics, coding, and reasoning with two base LLMs (GPT-OSS:120B and Gemini-3.1-Flash-Lite), with an additional evaluation on Humanity's Last Exam (HLE), and show strong performance together with improved token efficiency and robustness.
\end{itemize}

\section{Related Work}
\label{sec:related}

\subsection{LLM-Based Multi-Agent Systems}

LLM-based multi-agent systems have developed rapidly in recent years. Early work focused on role-playing and simulation settings~\citep{li2023camel, park2023generative}, while more recent systems have been applied to software engineering~\citep{hong2023metagpt, qian2024chatdev}, scientific reasoning~\citep{chen2023agentverse, ghareeb2025robin}, and general problem solving~\citep{qiu2025alita, wu2023autogen}. Alongside these developments, researchers have explored different ways for agents to interact. For example, LLM-Debate~\citep{du2023debate} uses structured argumentation to improve factual accuracy, and DyLAN~\citep{liu2024dylan} dynamically selects which agents should participate in a given task. Despite these advances, most existing systems still rely on predefined or heuristic communication patterns. In many cases, the interaction structure is fixed in advance or controlled at a global level, rather than adapting to the needs of a specific task. This suggests that communication is not just a supporting component, but a key factor that limits overall system performance.

\subsection{Communication Topology Optimization}

To address the limitations of fixed interaction patterns, recent work has started to explore adaptive communication topologies. Early approaches typically adopt simple structures such as chains~\citep{qian2024chatdev}, stars~\citep{wu2023autogen}, or fully connected graphs~\citep{du2023debate}. While these designs are easy to implement, they lack flexibility across tasks of varying difficulty. More recent methods attempt to learn communication structures automatically. G-Designer~\citep{zhang2024gdesigner} generates graphs using GNN-based models, AgentDropout~\citep{wang2025agentdropout} reduces unnecessary interactions through pruning, and GPTSwarm~\citep{zhuge2024gptswarm} optimizes agent connectivity using gradient-based methods. TopoDIM~\citep{sun2025topodim} introduces diverse interaction modes through a one-shot topology generator, while GTD~\citep{jiang2025gtd} leverages diffusion-based graph generation. MaAS~\citep{zhang2025maas} further formulates the problem as architecture search over agent systems. Although these approaches improve flexibility, they generally rely on centralized mechanisms that generate the communication structure at the system level. As a result, individual agents do not explicitly decide how to communicate during execution. Communication remains a global design choice, rather than something agents can adapt locally based on their own observations and experiences.

\subsection{Multi-Agent Corporative Reinforcement Learning}

Multi-agent reinforcement learning (MARL) provides a natural framework for modeling decentralized decision-making. The paradigm of centralized training with decentralized execution (CTDE)~\citep{oliehoek2016decpomdp} allows agents to use global information during training while relying only on local observations at test time. A large body of work focuses on value decomposition methods. VDN~\citep{sunehag2018vdn} assumes additive decomposition of the joint value function, QMIX~\citep{rashid2018qmix} introduces a monotonic mixing constraint to enable consistent decentralized optimization, and QPLEX~\citep{wang2021qplex} further increases expressiveness through more flexible decomposition. Prior work in MARL has also explored learned communication mechanisms, including differentiable message passing and attention-based communication~\citep{sukhbaatar2016commnet, foerster2016learning}. However, these approaches mainly focus on how information is transmitted, rather than how the communication structure itself is formed. In this work, we take a different perspective by treating communication topology as a decentralized decision problem. Each agent selects its own communication action, and the overall interaction graph emerges from the combination of these local decisions. This formulation connects topology learning with value decomposition methods and allows communication behavior to adapt dynamically during execution.

\begin{figure}[h]
    \centering
    
    \includegraphics[width=\linewidth]{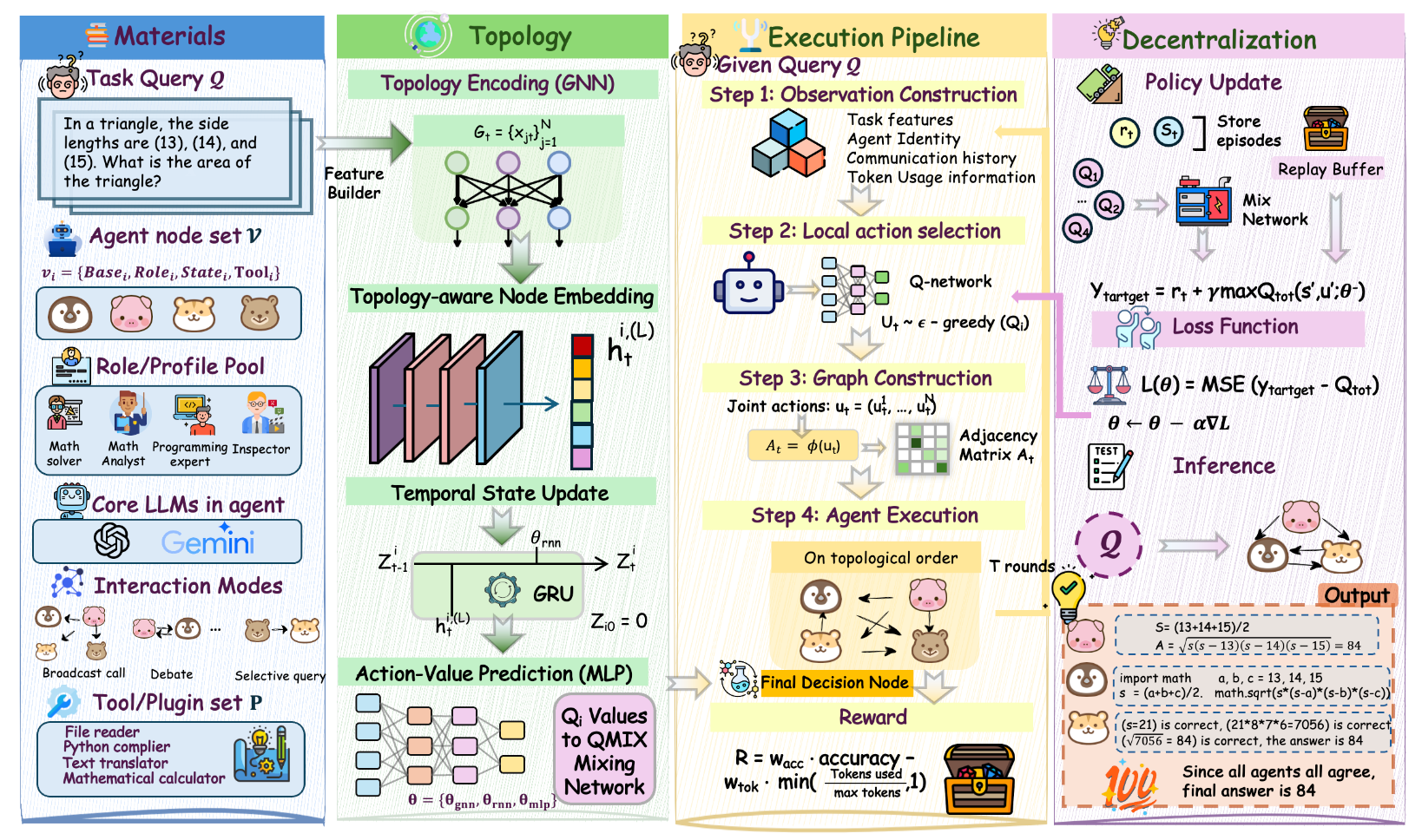}
    \caption{Overview of Agent Q-Mix. Agents select communication actions, induce a round-wise communication graph, exchange messages over that graph, and produce a final team output under centralized training with decentralized execution.}
    \label{fig:main_fig}
   
\end{figure}

\section{Preliminaries}
\label{sec:prelim}

\subsection{Networked Multi-Agent Decision Process}
\label{sec:mmdp}

We consider a cooperative system of $N$ agents interacting through a dynamic communication graph. We model this setting as a \emph{Networked Multi-Agent Markov Decision Process} (Networked MMDP), defined by the tuple
\[
\mathcal{M} = \bigl(\mathcal{N}, \mathcal{S}, \{\mathcal{A}_i\}_{i=1}^N, P, R, \gamma, \mathcal{G}\bigr),
\]
where $\mathcal{N}=\{1,\dots,N\}$ is the set of agents, $\mathcal{S}$ is the global state space, $\mathcal{A}_i$ is the action space of agent $i$, $P : \mathcal{S} \times \prod_{i=1}^N \mathcal{A}_i \to \Delta(\mathcal{S})$ is the transition kernel ($\Delta(\mathcal{S})$ denotes the set of probability distributions over $\mathcal{S}$), $R : \mathcal{S} \times \prod_{i=1}^N \mathcal{A}_i \to \mathbb{R}$ is the team reward, $\gamma \in [0,1)$ is the discount factor, and $\mathcal{G}_t = (\mathcal{N}, \mathcal{E}_t)$ denotes the communication graph at time $t$.

At time step $t$, agent $i$ acts based on its local action-observation history:
\[
\tau_t^i = (x_0^i, u_0^i, \ldots, x_t^i),
\]
where $x_t^i$ is the current observation and $u_t^i$ is the action selected at time step $t$. Let $\mathbf{u}_t = (u_t^1,\dots,u_t^N)$ denote the joint action. The objective is to maximize the expected discounted team return:
\begin{equation}
    J(\boldsymbol{\pi})
    =
    \mathbb{E}_{\boldsymbol{\pi}}
    \left[
        \sum_{t=0}^{\infty} \gamma^t \bar{R}_t
    \right],
    \qquad
    \bar{R}_t
    =
    \frac{1}{N}\sum_{i=1}^N R_i(s_t,\mathbf{u}_t),
    \label{eq:objective}
\end{equation}
where $\boldsymbol{\pi} = (\pi_1,\dots,\pi_N)$ is the joint policy.

\subsection{Graph Based Message Passing}
\label{sec:gnn_prelim}

To encode the current communication graph, we use a message-passing graph neural network (GNN)~\citep{kipf2017gcn}.
Given a graph $\mathcal{G} = (\mathcal{V}, \mathcal{E})$ with node features $\{x_v\}_{v \in \mathcal{V}}$, an $L$-layer GNN iteratively updates node representations through neighborhood aggregation.
For node $v$ at layer $l$,
\begin{equation}
    h_v^{(l)}
    =
    \mathrm{UPDATE}^{(l)}
    \Bigl(
        h_v^{(l-1)},
        \;
        \mathrm{AGG}^{(l)}
        \bigl(
            \{
            \mathrm{MSG}^{(l)}(h_u^{(l-1)})
            \mid
            u \in \mathcal{N}(v)
            \}
        \bigr)
    \Bigr),
    \label{eq:gnn}
\end{equation}
where $h_v^{(0)} = \mathrm{Linear}(x_v)$ initializes the node embedding.
Here, $\mathrm{MSG}^{(l)}$ transforms neighbor features, $\mathrm{AGG}^{(l)}$ aggregates messages, and $\mathrm{UPDATE}^{(l)}$ combines the aggregated neighborhood information with the current node state.

In our implementation, the update step is parameterized using a GRU cell~\citep{cho2014gru}, which helps preserve sequential dependencies during topology-aware message passing and avoids vanishing gradients in deep GNN stacks:
\begin{equation}
    h_v^{(l)}
    =
    \mathrm{GRU}^{(l)}
    \Bigl(
        \mathrm{AGG}^{(l)}(\cdot),
        \;
        h_v^{(l-1)}
    \Bigr).
    \label{eq:gnn_update}
\end{equation}
The final embedding $h_v^{(L)}$ captures both local node features and structural context from the communication graph.

\subsection{QMIX Cooperative Markov Decision Process}
\label{sec:qmix_prelim}

Directly optimizing the joint action-value function $\Qtot(\boldsymbol{\tau}, \mathbf{u})$ is difficult because the joint action space grows exponentially with the number of agents. QMIX~\citep{rashid2018qmix} addresses this by factorizing the joint value into individual agent utilities while preserving the \emph{Individual-Global-Max} (IGM) property.

Specifically, if the joint value function satisfies the monotonicity condition
\begin{equation}
    \frac{\partial \Qtot(\boldsymbol{\tau}, \mathbf{u}, s)}
    {\partial \Qi(\tau^i, u^i)}
    \ge 0,
    \qquad \forall i \in \mathcal{N},
    \label{eq:monotonicity}
\end{equation}
then maximizing the joint value is consistent with independently maximizing each individual utility:
\begin{equation}
    \arg\max_{\mathbf{u}} \Qtot(\boldsymbol{\tau}, \mathbf{u}, s)
    =
    \left(
    \arg\max_{u^1} Q_1(\tau^1,u^1),\;
    \ldots,\;
    \arg\max_{u^N} Q_N(\tau^N,u^N)
    \right).
    \label{eq:igm}
\end{equation}
This property is central to CTDE: the mixed joint value is used during training, while at test time each agent acts greedily with respect to its own local Q-function.

QMIX parameterizes the joint value using a mixing network whose weights and biases are generated from the global state $s$ through hypernetworks. Each hypernetwork is a single-hidden-layer MLP (hidden dimension 64) that maps $s$ to the mixing layer parameters. Let $\mathbf{q} = [Q_1,\dots,Q_N]^\top$ be the vector of individual Q-values. With $W_1 = \mathrm{abs}(\mathrm{HyperNet}_1(s))$, $b_1 = \mathrm{HyperNet}_{b_1}(s)$, $W_2 = \mathrm{abs}(\mathrm{HyperNet}_2(s))$, and $b_2 = \mathrm{HyperNet}_{b_2}(s)$, where $\mathrm{abs}(\cdot)$ denotes element-wise absolute value, the mixed value is defined in Equation~\ref{eq:mixing}:
\begin{equation}
    \Qtot(\boldsymbol{\tau}, \mathbf{u}, s)
    = W_2 \, \mathrm{ELU}(W_1 \mathbf{q} + b_1) + b_2.
    \label{eq:mixing}
\end{equation}
The element-wise absolute value ensures that the mixing weights are nonnegative, which enforces the monotonicity constraint in Equation~\ref{eq:monotonicity}.

Training minimizes a temporal-difference (TD) objective~\citep{mnih2015dqn}:
\begin{equation}
    \mathcal{L}(\theta)
    =
    \mathbb{E}_{\mathcal{D}}
    \Big[
        \big(
            y_t^{\mathrm{tot}}
            -
            \Qtot(\boldsymbol{\tau}_t,\mathbf{u}_t,s_t;\theta)
        \big)^2
    \Big],
    \label{eq:td_loss}
\end{equation}
with target
\begin{equation}
    y_t^{\mathrm{tot}}
    =
    \bar{R}_t
    +
    \gamma
    \max_{\mathbf{u}'}
    \Qtot(\boldsymbol{\tau}_{t+1}, \mathbf{u}', s_{t+1}; \theta^-),
    \label{eq:td_target}
\end{equation}
where $\theta^-$ denotes the target-network parameters~\citep{mnih2015dqn}.

\section{Agent Q-Mix}
\label{sec:method}

We now instantiate the Networked MMDP for communication-topology learning in LLM-based multi-agent systems, as illustrated in Figure~\ref{fig:main_fig}.
Our method consists of three components:
\textbf{(i)} a discrete communication action space,
\textbf{(ii)} a topology-aware agent Q-network, and
\textbf{(iii)} a multi-round execution procedure that converts per-agent actions into a communication graph and executes the resulting workflow.

\begin{figure}[h]
    \centering
    
    \includegraphics[width=\linewidth]{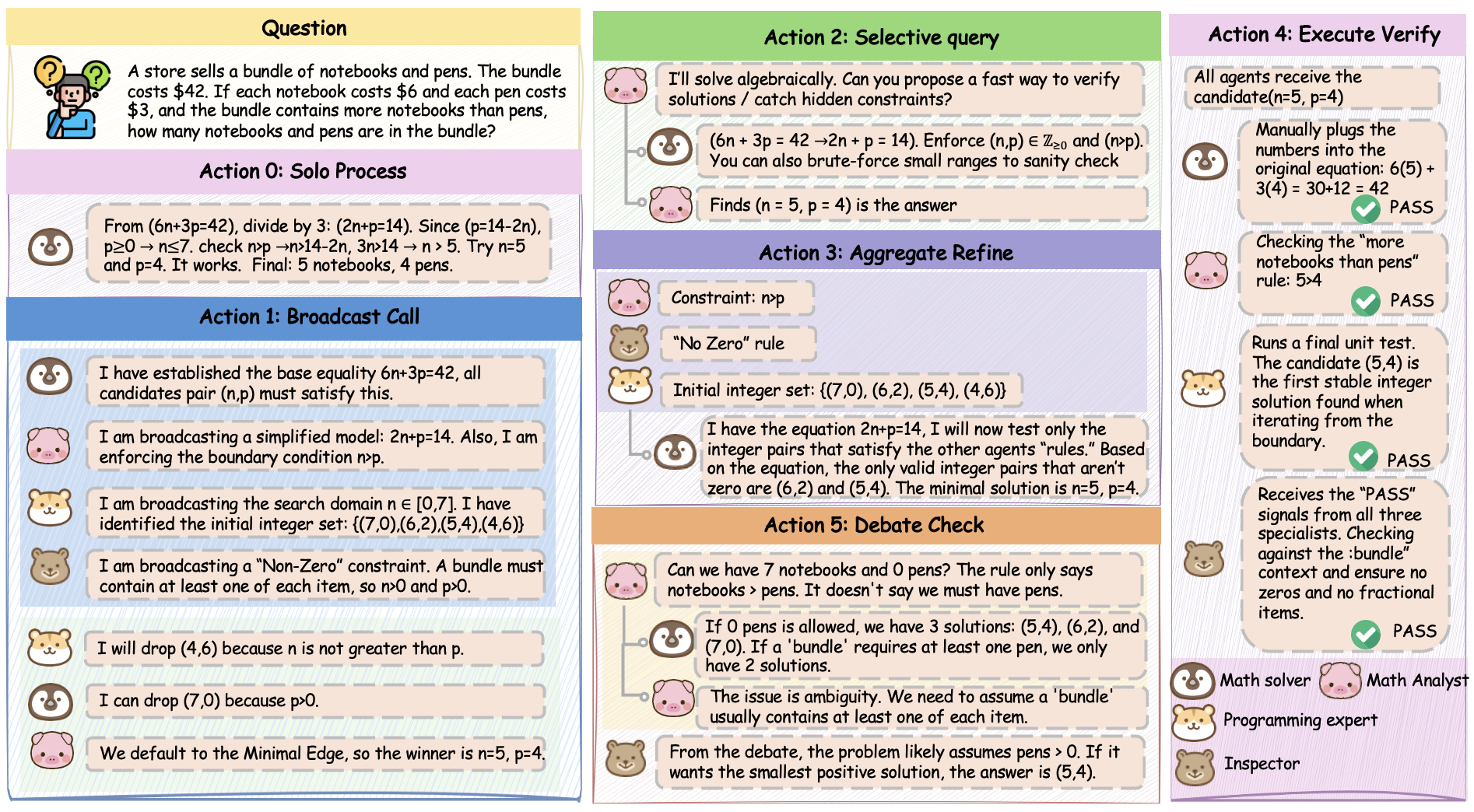}
    \caption{Illustration of the six discrete communication actions in the Agent Q-Mix action space (Section~\ref{sec:action_space}). Each panel depicts the induced graph structure for a representative multi-agent team. Arrows indicate the direction of information flow. From left to right: (a)~\texttt{solo\_process}---the agent works independently with no outgoing edges; (b)~\texttt{broadcast\_all}---the agent sends its message to all other agents; (c)~\texttt{selective\_query}---the agent sends a targeted query to one specific neighbor; (d)~\texttt{aggregate\_refine}---the agent collects responses from all others through incoming edges; (e)~\texttt{execute\_verify}---the agent forwards its output to the next agent for verification with minimal communication overhead; (f)~\texttt{debate\_check}---two agents engage in mutual exchange via bidirectional edges. These six primitives span a range of connectivity levels from isolated nodes to dense cliques, enabling the QMIX policy to compose task-adaptive topologies.}
    \label{fig:action_space}
   
\end{figure}

\subsection{Communication Action Space}
\label{sec:action_space}

Each agent selects from $|\mathcal{A}_i| = 6$ discrete communication actions at each round. Importantly, these actions do not directly dictate message content; rather, each action determines the \emph{structural connectivity} of the agent within the communication graph. The resulting graph topology then governs which agents exchange information during execution. In this sense, Agent Q-Mix learns the communication topology indirectly through the composition of per-agent structural decisions, rather than prescribing a fixed graph.
The joint action vector $\mathbf{u}_t = (u_t^1,\dots,u_t^N)$ is then mapped to an adjacency matrix $A_t$ that defines the communication graph.
The six actions are designed to cover a range of communication patterns observed in effective multi-agent workflows:

Figure~\ref{fig:action_space} summarizes the action space and its induced graph effects. Specifically, \texttt{solo\_process} creates no outgoing edges, \texttt{broadcast\_all} connects an agent to all other agents, \texttt{selective\_query} creates one targeted outgoing edge, \texttt{aggregate\_refine} collects information through incoming edges from all other agents, \texttt{execute\_verify} uses a minimal edge to the next agent, and \texttt{debate\_check} creates mutual edges with one partner.

This action space is both interpretable and expressive: it can recover common fixed topologies (e.g., all agents selecting \texttt{broadcast\_all} yields a fully connected graph; all selecting \texttt{solo\_process} yields no communication) while also permitting heterogeneous, task-adaptive patterns where different agents adopt different communication strategies within the same round.

\paragraph{Motivation and generality.}
We motivate the design of this six-action space from three perspectives.
\textbf{(i)~Grounding in multi-agent primitives.} Each action corresponds to a communication pattern well-studied in prior work: independent work (\texttt{solo\_process})~\citep{wu2023autogen}, round-robin debate (\texttt{debate\_check})~\citep{du2023debate}, broadcast-and-aggregate (\texttt{broadcast\_all}, \texttt{aggregate\_refine})~\citep{hong2023metagpt}, selective routing (\texttt{selective\_query})~\citep{liu2024dylan}, and tool-assisted verification (\texttt{execute\_verify})~\citep{qian2024chatdev}.
\textbf{(ii)~Graph-theoretic coverage.} The six actions span connectivity levels from isolated nodes (no edges) to full cliques, ensuring that the learnable space can express topologies of varying density.
\textbf{(iii)~Domain agnosticism.} The action space encodes no task-specific knowledge; it provides general structural primitives that the QMIX policy can compose. In Section~\ref{sec:main_results} and Appendix~\ref{app:policy}, we show that the same action set produces qualitatively different policies across coding, reasoning, and mathematics domains, suggesting that the framework extends to new task domains without redesigning the action space.

\subsection{Topology-Aware Agent Q-Network}
\label{sec:agent_network}

Each agent maintains an individual action-value function $\Qi(\tau_t^i,\cdot\,;\theta)$. The network first performs topology encoding: at communication round $t$, each agent has an observation $x_t^i$, and all observations $\{x_t^j\}_{j=1}^N$ are processed jointly by an $L$-layer GNN over the current communication graph $\mathcal{G}_t$ to produce topology-aware node embeddings. For agent $i$, we denote the final embedding by $h_t^{i,(L)}$. To capture multi-round dependencies, each agent maintains a recurrent hidden state that is updated with a GRU:
\begin{equation}
    z_t^i
    =
    \mathrm{GRU}\bigl(z_{t-1}^i,\, h_t^{i,(L)};\theta_{\mathrm{rnn}}\bigr),
    \qquad
    z_0^i = \mathbf{0}.
    \label{eq:temporal_gru}
\end{equation}
The updated temporal state is then passed to an MLP to produce Q-values over the communication action space:
\begin{equation}
    \Qi(\tau_t^i,\cdot\,;\theta)
    =
    \mathrm{MLP}\bigl(z_t^i;\theta_{\mathrm{mlp}}\bigr).
    \label{eq:qhead}
\end{equation}
The parameters
\[
\theta = \{\theta_{\mathrm{gnn}}, \theta_{\mathrm{rnn}}, \theta_{\mathrm{mlp}}\}
\]
are shared across agents, while agent identity is included in the observation features. The individual Q-values are then combined by the QMIX mixing network from Equation~\ref{eq:mixing}. The global state $s_t$ provided to the mixing network concatenates all agent observations together with graph-level statistics such as edge count, graph density, and the number of active agents. Gradients from the TD loss in Equation~\ref{eq:td_loss} backpropagate through the mixing network into all agent-side modules. We provide a formal convergence analysis of this training procedure in Appendix~\ref{app:convergence}.

\subsection{Multi-Round Execution Pipeline}
\label{sec:pipeline}

Given a task query $q$, Agent Q-Mix operates for $T$ communication rounds. At each round $t$, each agent forms an observation $x_t^i$ containing task features, agent identity, communication history, and token-usage information. The topology-aware agent Q-network computes $\Qi(\tau_t^i,\cdot)$, and agent $i$ selects an action $u_t^i$; during training, actions are selected using $\epsilon$-greedy exploration, while during deployment, agents act greedily. The joint action vector $\mathbf{u}_t = (u_t^1,\dots,u_t^N)$ is mapped to an adjacency matrix
\begin{equation}
    A_t = \phi(\mathbf{u}_t),
\end{equation}
which defines the communication graph for that round. Agents then execute according to the induced topology in topological order, where each agent conditions on information received from its spatial predecessors in the current round together with temporal context accumulated from previous rounds. After $T$ rounds, a final decision node aggregates the agents' outputs to produce the system answer. The full training-and-execution workflow is summarized in Algorithm~\ref{alg:training}. To balance task quality and token efficiency, we use the reward
\begin{equation}
    R
    =
    w_{\mathrm{acc}} \cdot \mathrm{accuracy}
    -
    w_{\mathrm{tok}} \cdot
    \min\!\left(
        \frac{\mathrm{tokens\_used}}{\mathrm{max\_tokens}},
        1
    \right),
    \label{eq:reward}
\end{equation}
where $\mathrm{max\_tokens} = 10{,}000$ is a fixed normalization budget that caps the token-cost ratio at~1, and $w_{\mathrm{acc}}$, $w_{\mathrm{tok}}$ are domain-specific reward weights (reported in Table~\ref{tab:hyperparams}).

To provide an example of the learned communication behavior, Figure~\ref{fig:case_study} shows a representative trajectory in which agents move from initial independent reasoning to targeted exchange and final aggregation. This behavior is consistent with the adaptive multi-round communication pipeline described above.

\begin{figure}[t]
\centering
\includegraphics[width=\linewidth]{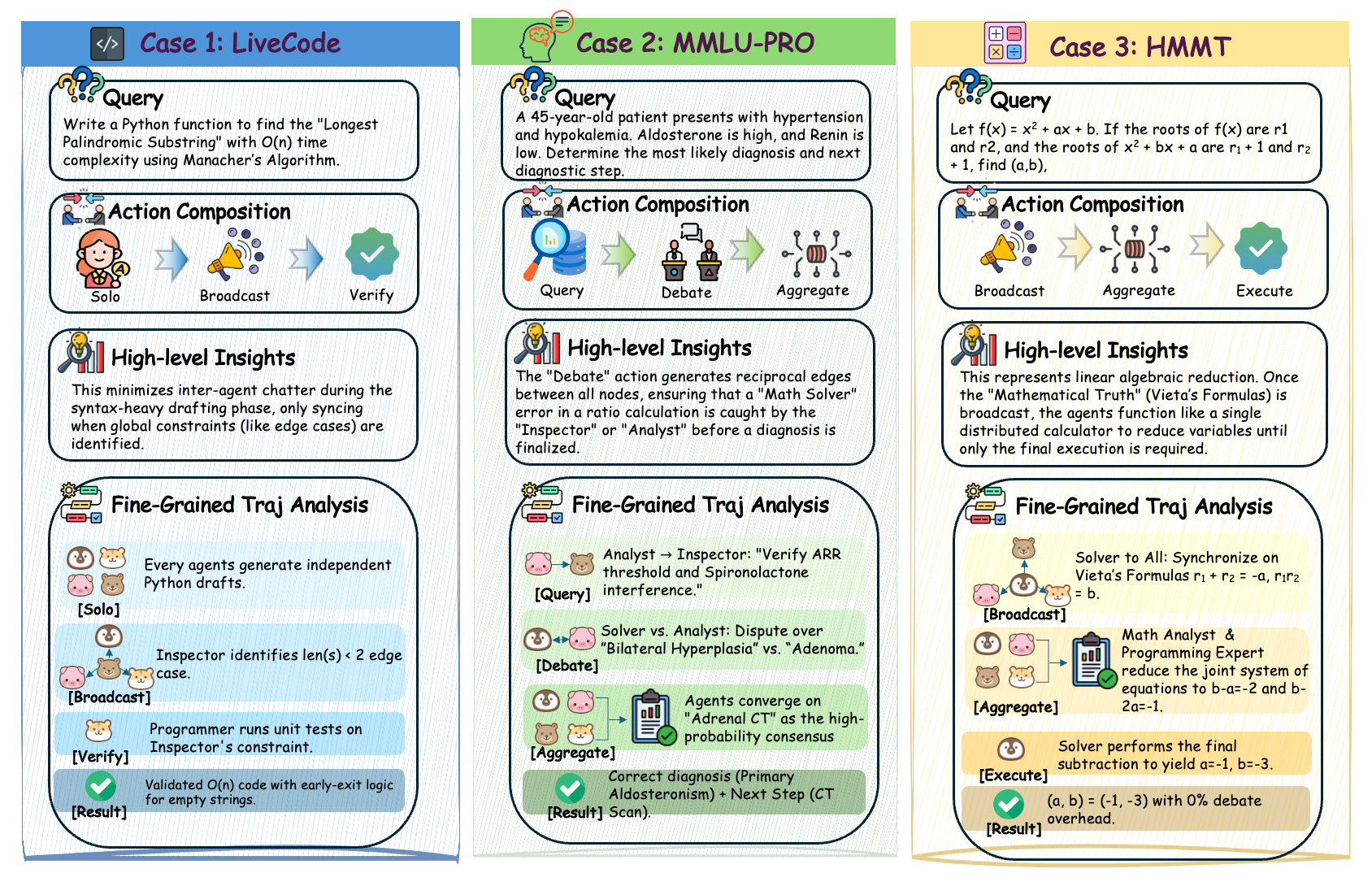}
\caption{Case study of Agent Q-Mix showing a representative communication and reasoning trajectory across rounds. In practice, mathematics tasks typically converge in $T=3$ communication rounds, while coding and reasoning tasks require $T=2$ rounds (see the ablation on communication rounds in Figure~\ref{fig:ablation}d and Appendix~\ref{app:training}).}
\label{fig:case_study}
\end{figure}

\section{Experiments}
\label{sec:experiments}

\subsection{Experimental Setup}
\label{sec:setup}

We evaluate Agent Q-Mix on two base LLMs (GPT-OSS:120B and Gemini-3.1-Flash-Lite) across seven held-out benchmarks spanning coding, reasoning, and mathematics. We compare against single-agent prompting, static multi-agent coordination, adaptive topology optimization methods, and commercial multi-agent frameworks. Figure~\ref{fig:hle} highlights an additional evaluation on Humanity's Last Exam (HLE) under the same setup as MMLU-Pro. To keep the main paper focused on empirical findings, we move detailed benchmark definitions, baseline composition, and implementation configuration to Appendix~\ref{app:exp_setup}, with complementary dataset statistics and hyperparameters reported in Appendix~\ref{app:data} and Appendix~\ref{app:training}. The full prompts used for each agent role are provided in Appendix~\ref{app:prompts}. This protocol isolates the effect of the proposed method in Section~\ref{sec:method}: any performance differences are attributable to learned communication actions and QMIX-based topology adaptation rather than changes in evaluation tasks.

\subsection{Main Results}
\label{sec:main_results}

Table~\ref{tab:main} presents the main accuracy comparison across two base LLMs: GPT-OSS:120B and Gemini-3.1-Flash-Lite.
All results are the median of three independent runs.

\begin{wrapfigure}{r}{0.50\linewidth}
\vspace{-2em}
\centering
\includegraphics[width=\linewidth]{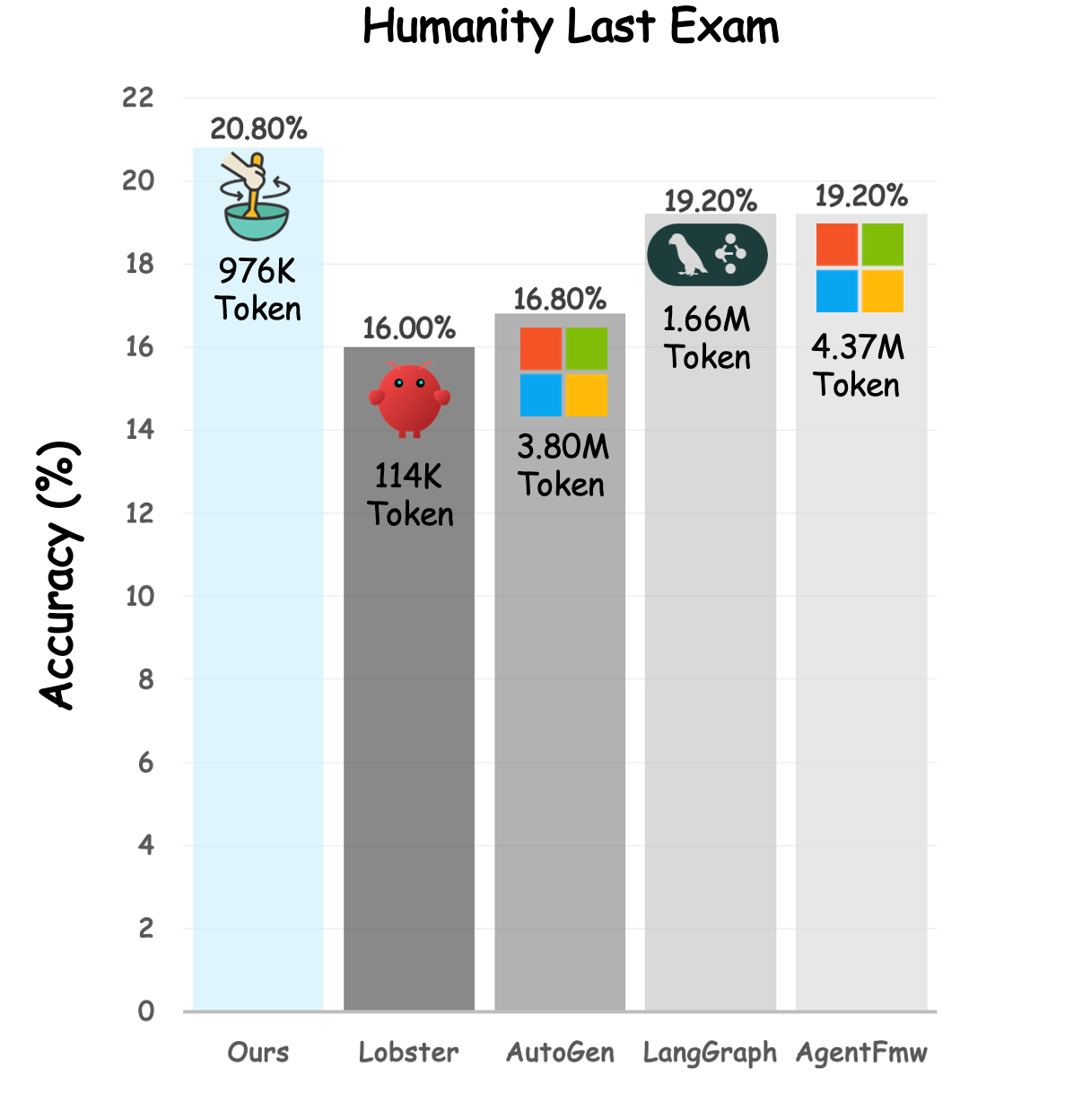}
\caption{\textbf{Humanity's Last Exam (HLE).} Accuracy (\%) on the first 250 MCQ items of HLE using Gemini-3.1-Flash-Lite. Agent Q-Mix (20.8\%) outperforms Microsoft Agent Framework (19.2\%), LangGraph (19.2\%), AutoGen, and Lobster, while using significantly lower tokens.}
\label{fig:hle}
\vspace{-3.0em}
\end{wrapfigure}

\begin{table*}[t]
\centering
\small
\setlength{\tabcolsep}{3.5pt}
\resizebox{\textwidth}{!}{%
\begin{tabular}{l cc c cccc c}
 \rowcolor{gray!10} \toprule
\multirow{2}{*}{\textbf{\textsc{Model/Methods}}}  
& \multicolumn{2}{c}{\textbf{Code}} & \textbf{Reasoning} & \multicolumn{4}{c}{\textbf{Mathematics}} & \\
\cmidrule(lr){2-3} \cmidrule(lr){4-4} \cmidrule(lr){5-8}
\rowcolor{gray!10} & LiveCode & HE & MMLU-Pro & AIME$_{25}$ & AIME$_{26}$ & HMMT & Beyond & \textbf{Avg} \\
\midrule
\rowcolor{gray!8}
\multicolumn{9}{c}{\textbf{Base Model: GPT-OSS:120B}} \\
\midrule
\multicolumn{9}{l}{\textit{Single-agent baseline}} \\
Base (direct)~\citep{agarwal2025gptoss} & 88.75 & 90.04 & 74.29 & 23.33 & 20.00 & 26.67 & 18.00 & 48.73 \\
\midrule
\multicolumn{9}{l}{\textit{Static multi-agent}} \\
LLM-Debate~\citep{du2023debate} & 92.00 & 89.63 & 78.57 & 30.00 & 36.67 & 33.33 & 28.00 & 55.46 \\
\midrule
\multicolumn{9}{l}{\textit{Adaptive topology methods}} \\
GPTSwarm~\citep{zhuge2024gptswarm} & 95.75 & 90.24 & 81.43 & 36.67 & 43.33 & 36.67 & 30.00 & 59.16 \\
AgentDropout~\citep{wang2025agentdropout} & 90.00 & 88.41 & 82.86 & 33.33 & 30.00 & 33.33 & 22.00 & 54.28 \\
G-Designer~\citep{zhang2024gdesigner} & 92.25 & 92.08 & 84.29 & 40.00 & 43.33 & 36.67 & 32.00 & 60.09 \\
MaAS~\citep{zhang2025maas} & 94.50 & 92.07 & 84.29 & 36.67 & 33.33 & 33.33 & 27.00 & 57.31 \\
GTD~\citep{jiang2025gtd} & 93.75 & 91.46 & 85.71 & 33.33 & 40.00 & \underline{43.33} & 35.00 & 60.37 \\
TopoDIM~\citep{sun2025topodim} & 94.00 & 95.73 & \underline{88.57} & 36.67 & 26.67 & 40.00 & 34.00 & 59.38 \\
\midrule
\multicolumn{9}{l}{\textit{Commercial Framework (multi-agent mode)}} \\
Lobster~\citep{openclaw2025lobster} & 96.00 & \underline{96.34} & 87.14 & \underline{63.33} & 60.00 & 6.67 & 15.00 & 60.64 \\
LangGraph~\citep{langgraph2024} & 93.50 & 93.90 & 45.86 & 46.67 & 60.00 & 30.00 & 25.00 & 56.42 \\
AutoGen~\citep{wu2023autogen} & \underline{100.00} & 95.12 & 81.43 & 53.33 & \underline{70.00} & 43.33 & \underline{37.00} & \underline{68.60} \\
Agent Framework~\citep{microsoft2025agentframework} & 98.25 & 95.73 & 80.00 & 53.33 & 63.33 & 30.00 & 32.00 & 64.66 \\
\midrule
\rowcolor{lightpurple} \textbf{Agent Q-Mix} & \textbf{100.00} & \textbf{97.56} & \textbf{92.86} & \textbf{63.33} & 60.00 & \textbf{53.33} & \textbf{42.00} & \textbf{72.73} \\
\midrule
\midrule
\rowcolor{gray!8}
\multicolumn{9}{c}{\textbf{Base Model: Gemini-3.1-Flash-Lite}} \\
\midrule
\multicolumn{9}{l}{\textit{Single-agent baseline}} \\
Base (direct)~\citep{comanici2025gemini} & 96.50 & 92.07 & 81.43 & 30.00 & 13.33 & 33.33 & 22.00 & 52.67 \\
\midrule
\multicolumn{9}{l}{\textit{Static multi-agent}} \\
LLM-Debate~\citep{du2023debate} & 98.00 & 92.68 & 82.86 & 33.33 & 20.00 & 33.33 & 18.00 & 54.03 \\
\midrule
\multicolumn{9}{l}{\textit{Adaptive topology methods}} \\
GPTSwarm~\citep{zhuge2024gptswarm} & 99.00 & 93.70 & 82.86 & 36.67 & 40.00 & 36.47 & 20.00 & 58.39 \\
AgentDropout~\citep{wang2025agentdropout} & 96.25 & 91.46 & 74.29 & 26.67 & 33.33 & 36.67 & 29.00 & 55.38 \\
G-Designer~\citep{zhang2024gdesigner} & 98.75 & 93.90 & 78.57 & 33.33 & 36.67 & 43.33 & 27.00 & 58.79 \\
MaAS~\citep{zhang2025maas} & 100.00 & 92.68 & 85.71 & 36.67 & 40.00 & \underline{46.67} & 27.00 & 61.25 \\
GTD~\citep{jiang2025gtd} & 97.75 & 92.68 & 87.14 & 40.00 & 46.67 & 43.33 & 31.00 & 62.65 \\
TopoDIM~\citep{sun2025topodim} & 98.25 & 93.90 & 85.71 & 40.00 & 43.33 & \underline{46.67} & 32.00 & 62.84 \\
\midrule
\multicolumn{9}{l}{\textit{Commercial Framework (multi-agent mode)}} \\
Lobster~\citep{openclaw2025lobster} & 100.00 & \underline{97.56} & \underline{91.43} & 36.67 & 36.67 & 26.33 & 24.00 & 58.95 \\
LangGraph~\citep{langgraph2024} & 100.00 & 95.12 & \underline{91.43} & \underline{46.67} & 36.67 & 26.33 & 24.00 & 60.03 \\
AutoGen~\citep{wu2023autogen} & 100.00 & 96.95 & 88.57 & 40.00 & \underline{60.00} & 33.33 & 25.00 & 63.41 \\
Agent Framework~\citep{microsoft2025agentframework} & 100.00 & 94.51 & 90.00 & 40.00 & 56.67 & 40.00 & \underline{33.00} & \underline{64.88} \\
\midrule
\rowcolor{lightpurple}  \textbf{Agent Q-Mix} & \textbf{100.00} & 95.73 & 88.57 & 40.00 & \textbf{60.00} & \textbf{50.00} & \textbf{34.00} & \textbf{66.90} \\
\bottomrule
\end{tabular}
}
\caption{Accuracy (\%) on seven benchmarks compared to base (single-agent). \textbf{Bold}: best overall per model. \underline{Underline}: best among prior methods.}
\label{tab:main}
\end{table*}

As shown in Table~\ref{tab:main}, Agent Q-Mix yields the strongest overall performance on GPT-OSS:120B, achieving 72.73\% average accuracy across all seven benchmarks. This improves on the best adaptive-topology baseline (GTD, 60.37\%) by +12.36 points and on the strongest commercial framework baseline (AutoGen, 68.60\%) by +4.13 points. On Gemini-3.1-Flash-Lite, Agent Q-Mix remains highly competitive and achieves the strongest reported mathematics profile (AIME$_{25}$: 40.00, AIME$_{26}$: 60.00, HMMT: 50.00, Beyond-AIME: 34.00), indicating that the same topology-learning mechanism transfers across model families.

The largest gains appear on mathematics benchmarks, where coordinated multi-round deliberation is most important. In particular, Agent Q-Mix reaches 53.33\% on HMMT~2025, improving over the next best prior method (43.33\%) by +10.00 points, and reaches 42.00\% on Beyond-AIME, outperforming GTD (35.00\%) and AutoGen (37.00\%). These improvements are consistent with the method design in Section~\ref{sec:method}: the topology-aware GNN representation and recurrent state allow agents to adapt communication patterns over rounds rather than relying on a fixed interaction template.

Agent Q-Mix also preserves strong performance on coding and reasoning tasks, achieving 100.00\% on LiveCodeBench~(v6), 97.56\% on HumanEval, and 92.86\% on MMLU-Pro with GPT-OSS:120B. This supports the central claim of the paper: QMIX-based decentralized action selection can increase coordination quality without sacrificing performance on domains where single-agent performance is already high. We provide a qualitative analysis with representative examples from each domain in Appendix~\ref{app:qualitative}.

\subsection{Robustness to Adversarial Agents}
\label{sec:robustness}

To evaluate robustness, we replace one of the three agents with an adversarial agent that deliberately provides incorrect answers on MMLU-Pro.
Table~\ref{tab:robustness} reports accuracy before and after the adversarial injection.

\begin{table}[h]
\centering
\includegraphics[width=\linewidth]{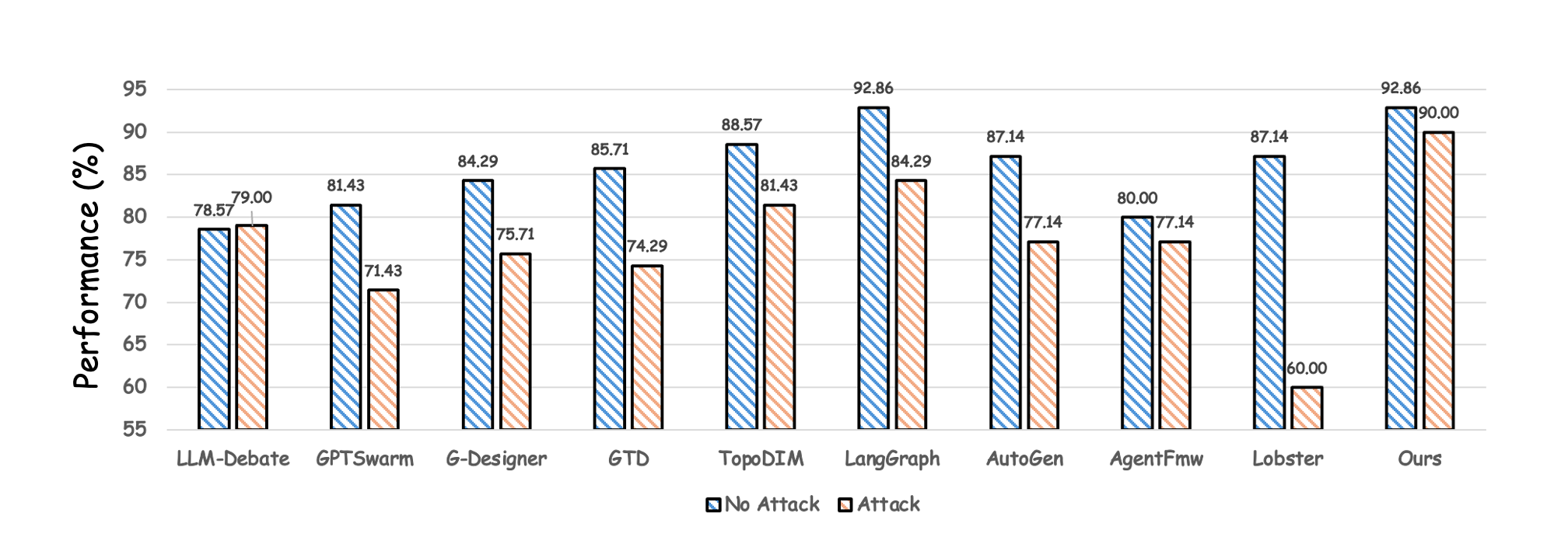}
\caption{Robustness on MMLU-Pro (GPT-OSS:120B). One adversarial agent is injected. $\Delta$: accuracy drop.}
\label{tab:robustness}
\end{table}

Table~\ref{tab:robustness} indicates that Agent Q-Mix is the most robust method under adversarial-agent injection on MMLU-Pro. Accuracy drops by only 2.86 points (92.86\% to 90.00\%), compared with larger degradations for prior multi-agent methods (e.g., 8.57 for LLM-Debate and 10.00 for GPTSwarm and AutoGen). This behavior is consistent with the method’s adaptive-graph mechanism: when one agent becomes unreliable, the learned policy can down-weight that agent’s influence by selecting communication actions that reduce harmful edge propagation.

Overall, the robustness results complement the accuracy and efficiency findings: the same topology-learning objective that improves average performance also increases fault tolerance under distribution shift in agent quality.

\subsection{Token Efficiency}
\label{sec:efficiency}

A key advantage of learned topologies is that agents can communicate less on easy problems and more on hard ones.
Table~\ref{tab:token} reports total token usage on three representative benchmarks.

\begin{table}[t]
\centering
\includegraphics[width=\linewidth]{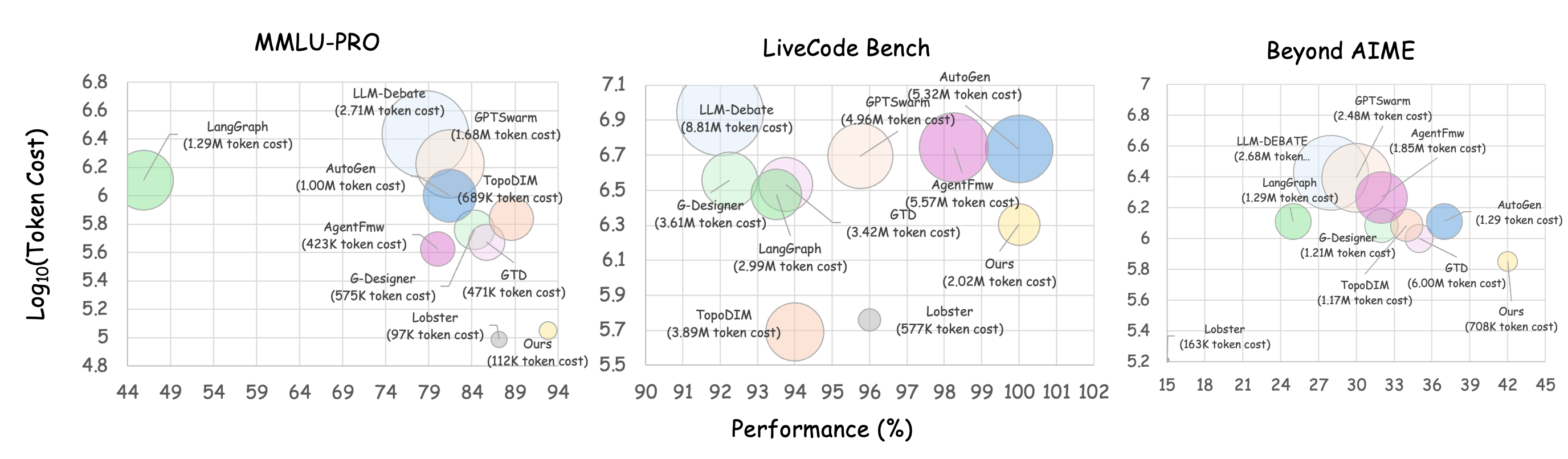}
\caption{Token usage comparison on GPT-OSS:120B (total tokens across all samples). Lower is better at comparable accuracy.}
\label{tab:token}
\end{table}

Table~\ref{tab:token} shows that Agent Q-Mix is substantially more token-efficient than most multi-agent baselines while remaining highly accurate. On MMLU-Pro, Agent Q-Mix uses 112K tokens, which is only 15\% above the single-agent Lobster baseline (97K) and 4--24$\times$ lower than other multi-agent methods (471K--2.71M). On Beyond-AIME, Agent Q-Mix uses 708K tokens versus 1.00M--2.68M for representative adaptive and static baselines, while delivering stronger accuracy in Table~\ref{tab:main}.

These efficiency gains are aligned with the learning objective in Equation~\ref{eq:reward} and the action semantics in Section~\ref{sec:action_space}. Because communication is selected as a learned per-agent action, the policy can allocate dense communication only when needed and default to sparse coordination on easier instances.

\section{Conclusion}
\label{sec:conclusion}

We presented Agent Q-Mix, a framework that casts communication-topology learning in LLM-based multi-agent systems as a cooperative MARL problem and solves it through QMIX-based value decomposition.
By combining topology-aware GNN encoding, temporal GRU memory, and monotonic mixing, the method supports centralized training together with decentralized communication-action selection at deployment time.
Each agent independently chooses from a set of interpretable communication actions, and the resulting joint decisions induce the team communication graph across multiple rounds.

\bibliography{colm2024_conference}
\bibliographystyle{plainnat}

\newpage
\appendix
\section{Experimental Setup and Baselines}
\label{app:exp_setup}

We evaluate Agent Q-Mix with two base LLMs: GPT-OSS:120B~\citep{agarwal2025gptoss} and Gemini-3.1-Flash-Lite. The evaluation suite covers seven core benchmarks across three domains: coding (LiveCodeBench v6~\citep{jain2024livecodebench}, HumanEval~\citep{chen2021codex}), reasoning (MMLU-Pro~\citep{wang2024mmlupro}), and mathematics (AIME$_{2025}$~\citep{aime2025}, AIME$_{2026}$~\citep{aime2025}, Beyond-AIME~\citep{beyondaime2025}, HMMT$_{2025}$~\citep{hmmt2025}). We additionally evaluate on Humanity's Last Exam (HLE), a challenging multi-domain benchmark, using the first 250 MCQ items under the same multi-agent configuration as MMLU-Pro.

We compare against four baseline groups: (i) a single-agent baseline (vanilla direct prompting), (ii) a static multi-agent baseline (LLM-Debate~\citep{du2023debate}), (iii) adaptive topology methods (GPTSwarm~\citep{zhuge2024gptswarm}, AgentDropout~\citep{wang2025agentdropout}, G-Designer~\citep{zhang2024gdesigner}, MaAS~\citep{zhang2025maas}, TopoDIM~\citep{sun2025topodim}, GTD~\citep{jiang2025gtd}), and (iv) commercial multi-agent frameworks.

For implementation, Agent Q-Mix uses three domain-specific teams, each with two domain specialists and one AnalyzeAgent, followed by a FinalRefer decision node. The main network uses a 2-layer GNN (hidden size 128), a GRU (hidden size 128), and a QMIX hypernetwork output dimension of 64. We train with Adam (learning rate $5\times10^{-4}$), discount factor $\gamma=0.99$, replay-buffer size 5000, and batch size 8, with $\epsilon$ annealed from 1.0 to 0.05. We train on 15 examples per domain for 50 episodes and evaluate on held-out benchmarks, using 3 communication rounds for mathematics and 2 rounds for coding and reasoning.

\section{Algorithm Pseudocode}
\label{app:algorithm}

\begin{algorithm}[h]
\caption{Agent Q-Mix Training}
\label{alg:training}
\begin{algorithmic}[1]
\REQUIRE Dataset $\mathcal{D}_{\text{train}}$, episodes $M$, rounds $T$, agents $N$
\ENSURE Trained agent network $\theta$, mixing network $\psi$
\STATE Initialize agent Q-network $Q_\theta$, mixing network $f_\psi$, and target networks $\theta^- \leftarrow \theta$, $\psi^- \leftarrow \psi$
\STATE Initialize replay buffer $\mathcal{B}$ and exploration rate $\epsilon \leftarrow 1.0$
\FOR{episode $= 1, \ldots, M$}
    \STATE Sample task $q \sim \mathcal{D}_{\text{train}}$
    \STATE Initialize GRU hidden states $z_0^i \leftarrow \mathbf{0}$ for all agents $i \in \{1,\ldots,N\}$
    \STATE Initialize adjacency matrix $A_0 \leftarrow I$ \hfill \textit{(identity = no initial communication)}
    \FOR{round $t = 1, \ldots, T$}
        \STATE $\{x_t^i\}_{i=1}^N \leftarrow \textsc{GetObservations}(q, \mathcal{G}_t)$ \hfill \textit{($\mathcal{G}_t = (\mathcal{N}, \mathcal{E}_t)$ is the communication graph at round $t$)}
        \STATE $\{h_t^{i,(L)}\}_{i=1}^N \leftarrow \text{GNN}(\{x_t^j\}, A_t;\theta_{\text{gnn}})$ \hfill \textit{(topology-aware node embeddings, $L$-layer GNN)}
        \FOR{agent $i = 1, \ldots, N$}
            \STATE $z_t^i \leftarrow \text{GRU}(z_{t-1}^i, h_t^{i,(L)};\theta_{\text{rnn}})$ \hfill \textit{(temporal hidden state, Eq.~\ref{eq:temporal_gru})}
            \STATE $Q_i(\tau_t^i,\cdot) \leftarrow \text{MLP}(z_t^i;\theta_{\text{mlp}})$ \hfill \textit{(per-action Q-values, Eq.~\ref{eq:qhead})}
            \STATE $u_t^i \leftarrow \epsilon\text{-greedy}(Q_i)$
        \ENDFOR
        \STATE $A_{t+1} \leftarrow \phi(\mathbf{u}_t)$ \hfill \textit{($\phi$: action-to-adjacency mapping, Section~\ref{sec:action_space})}
        \STATE Construct global state $s_t \leftarrow [\{x_t^i\}_{i=1}^N, |\mathcal{E}_t|, \text{density}(\mathcal{G}_t), N]$
        \STATE Execute agents in topological order induced by $A_{t+1}$
    \ENDFOR
    \STATE Compute final task accuracy and token usage
    \STATE $R \leftarrow w_{\text{acc}} \cdot \text{acc} - w_{\text{tok}} \cdot \min(\text{tokens\_used}/\text{max\_tokens},\, 1)$ \hfill \textit{(Eq.~\ref{eq:reward})}
    \STATE Store episode transition $(\{x_t, u_t, R, A_t, s_t\}_{t=1}^T)$ in $\mathcal{B}$
    \STATE Sample a batch from $\mathcal{B}$ and compute the TD loss (Eq.~\ref{eq:td_loss})
    \STATE Update $\theta$ and $\psi$ by gradient descent
    \STATE Periodically update target networks: $\theta^- \leftarrow \theta$, $\psi^- \leftarrow \psi$
    \STATE Decay $\epsilon$
\ENDFOR
\end{algorithmic}
\end{algorithm}

\section{Data Statistics}
\label{app:data}

Table~\ref{tab:datasets} summarizes all datasets used for training and evaluation.

\begin{table}[h]
\centering
\caption{Dataset statistics.}
\label{tab:datasets}
\small
\begin{tabular}{llccc}
\toprule
\textbf{Dataset} & \textbf{Domain} & \textbf{Split} & \textbf{Samples} & \textbf{Phase} \\
\midrule
LCB TestGen & Coding & test & 15 & Train \\
MMLU-Pro & Reasoning & val & 15 & Train \\
AIME 2024 & Math & train & 15 & Train \\
\midrule
LiveCodeBench & Coding & test & 400 & Eval \\
HumanEval & Coding & test & 164 & Eval \\
MMLU-Pro & Reasoning & val & 70 & Eval \\
AIME 2025 & Math & train & 30 & Eval \\
AIME 2026 & Math & train & 30 & Eval \\
Beyond-AIME & Math & test & 100 & Eval \\
HMMT 2025 & Math & train & 30 & Eval \\
\bottomrule
\end{tabular}
\end{table}

We train on only 15 examples per domain to demonstrate that the QMIX policy generalizes from minimal supervision.
Evaluation benchmarks are held out and span a range of difficulty levels.

\section{Training Details}
\label{app:training}

\paragraph{Hyperparameters.}
Table~\ref{tab:hyperparams} lists the full hyperparameter configuration.

\begin{table}[h]
\centering
\caption{Hyperparameter settings.}
\label{tab:hyperparams}
\small
\begin{tabular}{lc}
\toprule
\textbf{Hyperparameter} & \textbf{Value} \\
\midrule
GNN layers & 2 \\
GNN hidden dimension & 128 \\
GRU hidden dimension & 128 \\
QMIX mixing hidden dim & 64 \\
Hyper-network hidden dim & 64 \\
Optimizer & Adam \\
Learning rate & $5 \times 10^{-4}$ \\
Discount factor ($\gamma$) & 0.99 \\
Replay buffer capacity & 5000 \\
Batch size & 8 \\
Gradient clip norm & 10.0 \\
Target network update interval & 200 steps \\
$\epsilon$ start / end & 1.0 / 0.05 \\
$\epsilon$ annealing & Linear over all episodes \\
\midrule
\multicolumn{2}{l}{\textit{Domain-specific reward weights}} \\
$w_{\mathrm{acc}}$ (GPT-OSS:120B) & 1.25 \\
$w_{\mathrm{tok}}$ (GPT-OSS:120B) & 0.10 \\
$w_{\mathrm{acc}}$ (Gemini-3.1-Flash-Lite) & 1.50 \\
$w_{\mathrm{tok}}$ (Gemini-3.1-Flash-Lite) & 0.075 \\
$\mathrm{max\_tokens}$ (normalization budget) & 10{,}000 \\
\midrule
\multicolumn{2}{l}{\textit{Communication rounds}} \\
Mathematics benchmarks & 3 \\
Coding and reasoning benchmarks & 2 \\
\bottomrule
\end{tabular}
\end{table}

\paragraph{Agent teams.}
Each domain uses a team of three agents plus one FinalRefer decision node. The agent roles are differentiated through domain-specific system prompts (full prompts are provided in Appendix~\ref{app:prompts}):
\begin{itemize}
    \item \textbf{Coding:} CodeWriter $\times 2$ + AnalyzeAgent. Each CodeWriter cycles through specialized roles (Project Manager, Algorithm Designer, Programming Expert, Test Analyst, Bug Fixer), each with a distinct system prompt that emphasizes design, implementation, testing, or debugging. The AnalyzeAgent serves as a critical reviewer that identifies potential errors in the proposed code.
    \item \textbf{Reasoning:} ReasoningAgent $\times 2$ + AnalyzeAgent. ReasoningAgents are prompted as ``world-class reasoning experts'' instructed to think step by step and consider multiple perspectives. For MMLU-Pro and HLE, agents additionally receive multiple-choice formatting constraints. The AnalyzeAgent identifies logical gaps and potential errors.
    \item \textbf{Mathematics:} MathSolver $\times 2$ + AnalyzeAgent. MathSolvers cycle through roles including Math Solver, Mathematical Analyst, Programming Expert (for code-based verification), and Inspector. Each role receives few-shot exemplars. The AnalyzeAgent cross-checks solutions and computations.
\end{itemize}
The FinalRefer decision node receives all agents' outputs and synthesizes the final answer. Its system prompt combines a domain-specific decision role (e.g., ``top decision-maker good at analyzing mathematical problems'') with output-format constraints.

\paragraph{Training procedure.}
We train three separate QMIX checkpoints (one per domain) using 50 episodes each on the 15-example training sets.
Training takes approximately 30--60 minutes per domain on a single CPU, as the primary computational cost is the LLM API calls rather than the QMIX network updates.
The trained checkpoints are then evaluated on all seven held-out benchmarks.

\section{Learned Policy Analysis}
\label{app:policy}

\paragraph{Action distribution.}
We analyze the actions selected by the trained QMIX policy across benchmarks.
On mathematics benchmarks (AIME, HMMT, Beyond-AIME), the policy predominantly selects \texttt{broadcast\_all} (action 1), indicating that these hard reasoning tasks benefit from full inter-agent communication.
On LiveCodeBench, where the base model already achieves near-perfect accuracy, the policy frequently selects \texttt{solo\_process} (action 0), avoiding unnecessary token expenditure.
On MMLU-Pro, the policy shows a mixture of \texttt{selective\_query} (action 2) and \texttt{debate\_check} (action 5), suggesting that targeted information exchange and adversarial verification are most useful for multi-choice reasoning.

\paragraph{Topology adaptation across rounds.}
The QMIX policy adapts its topology between communication rounds.
In round~1, agents typically adopt sparser topologies (\texttt{solo\_process} or \texttt{selective\_query}) to independently form initial estimates.
In subsequent rounds, the policy shifts toward denser patterns (\texttt{broadcast\_all} or \texttt{aggregate\_refine}) to consolidate and refine answers through cross-agent discussion.
This adaptive behavior naturally emerges from the reward signal without explicit round-level supervision.

\paragraph{Robustness mechanism.}
Under adversarial injection (Section~\ref{sec:robustness}), the QMIX policy learns to isolate the unreliable agent.
Analysis of the learned adjacency matrices shows that the adversarial agent's outgoing edges are significantly reduced compared to the non-adversarial setting, while the remaining agents increase mutual communication.
This selective isolation explains the small accuracy drop (2.86\%) observed in the robustness experiments.

\section{Convergence Analysis}
\label{app:convergence}

We establish the theoretical properties of the Agent Q-Mix training procedure.
The analysis proceeds in three steps: (1) we show that the monotonic mixing network preserves the Individual-Global-Max (IGM) property, guaranteeing consistency between decentralized and centralized action selection; (2) we prove that the QMIX TD learning objective converges under standard assumptions; and (3) we show that the resulting greedy policy is monotonically improving.

\subsection{IGM Consistency}

We first recall the IGM principle and show that the QMIX monotonic mixing network satisfies it.

\begin{definition}[Individual-Global-Max~\citep{rashid2018qmix}]
A joint action-value function $\Qtot(\boldsymbol{\tau},\mathbf{u},s)$ satisfies the \emph{IGM property} with respect to individual utilities $\{Q_i(\tau^i,u^i)\}_{i=1}^N$ if
\begin{equation}
    \arg\max_{\mathbf{u}} \Qtot(\boldsymbol{\tau},\mathbf{u},s)
    =
    \Bigl(
    \arg\max_{u^1}Q_1(\tau^1,u^1),\;\ldots,\;\arg\max_{u^N}Q_N(\tau^N,u^N)
    \Bigr).
    \label{eq:igm_def}
\end{equation}
\end{definition}

\begin{theorem}[Monotonicity $\Rightarrow$ IGM]
\label{thm:igm}
Let $\Qtot(\boldsymbol{\tau},\mathbf{u},s) = f_\psi\bigl(Q_1(\tau^1,u^1),\ldots,Q_N(\tau^N,u^N);\,s\bigr)$ where $f_\psi$ is the QMIX mixing network parameterized by $\psi$.
If the mixing weights satisfy $W_k = |\mathrm{HyperNet}_k(s)| \ge 0$ for all layers $k$, then
\[
\frac{\partial \Qtot}{\partial Q_i} \ge 0, \quad \forall\, i \in \{1,\ldots,N\},
\]
and consequently $\Qtot$ satisfies the IGM property.
\end{theorem}

\begin{proof}
Recall the mixing network architecture from Equation~\ref{eq:mixing}:
\[
    \Qtot
    =
    W_2^\top \,\mathrm{ELU}(W_1 \mathbf{q} + b_1) + b_2,
\]
where $\mathbf{q} = [Q_1,\ldots,Q_N]^\top$ and $W_1,W_2$ are produced by hypernetworks followed by element-wise absolute value, so $W_1 \ge 0$ and $W_2 \ge 0$ entry-wise.

\textit{Step 1.}
Let $\mathbf{h} = W_1 \mathbf{q} + b_1$.
Then $\frac{\partial h_j}{\partial Q_i} = [W_1]_{ji} \ge 0$ for all $j, i$.

\textit{Step 2.}
The ELU activation $\sigma(x) = x$ for $x \ge 0$ and $\sigma(x) = e^x - 1$ for $x < 0$.
Its derivative satisfies $\sigma'(x) > 0$ for all $x \in \mathbb{R}$.
Therefore $\frac{\partial \sigma(h_j)}{\partial h_j} > 0$.

\textit{Step 3.}
By the chain rule,
\[
\frac{\partial \Qtot}{\partial Q_i}
=
\sum_j [W_2]_j \cdot \sigma'(h_j) \cdot [W_1]_{ji}
\ge 0,
\]
since every factor in each summand is nonneg.

\textit{Step 4.}
For any two joint actions $\mathbf{u}, \mathbf{u}'$ that differ only in agent $i$'s component, $\Qtot(\ldots,u^i,\ldots) \ge \Qtot(\ldots,u'^i,\ldots)$ if and only if $Q_i(\tau^i,u^i) \ge Q_i(\tau^i,u'^i)$.
Therefore the global argmax decomposes into independent per-agent argmaxes, which is exactly the IGM condition.
\end{proof}

\paragraph{Implication for Agent Q-Mix.}
Since our topology-aware agent Q-networks produce the individual $Q_i$ values and these are combined through the QMIX mixing network with nonneg weights, Theorem~\ref{thm:igm} guarantees that at deployment time, each agent can independently select $u^i_* = \arg\max_{u^i} Q_i(\tau^i,u^i)$ and the resulting joint action $\mathbf{u}_*$ maximizes $\Qtot$.
This is the formal basis for decentralized communication-action selection in Agent Q-Mix.

\subsection{Convergence of QMIX TD Learning}

We now give a convergence result in the standard realizable tabular regime, where exact guarantees are available.
The proof combines Bellman contraction with classical stochastic-approximation analysis for Q-learning~\citep{sutton2018rl}.

\begin{assumption}[Realizable tabular setting]
\label{assum:standard}
The following conditions hold:
\begin{enumerate}
    \item[(A1)] \textbf{Finite domain:} The centralized training process induces a finite joint state-action space.
    \item[(A2)] \textbf{Bounded rewards:} $|\bar{R}_t| \le R_{\max}$ for all $t$.
    \item[(A3)] \textbf{Sufficient exploration:} Every joint state-action pair is visited infinitely often.
    \item[(A4)] \textbf{Robbins--Monro step sizes:} For each joint pair, $\sum_t \alpha_t = \infty$ and $\sum_t \alpha_t^2 < \infty$.
    \item[(A5)] \textbf{QMIX realizability:} There exist parameters $(\theta^*,\psi^*)$ such that $\Qtot(\cdot;\theta^*,\psi^*) = Q^*_{\mathrm{tot}}$ and the monotonicity condition of Theorem~\ref{thm:igm} holds.
\end{enumerate}
\end{assumption}

\begin{theorem}[Convergence of QMIX TD Iterates]
\label{thm:convergence}
Under Assumption~\ref{assum:standard}, define
\[
Q_k(\boldsymbol{\tau},\mathbf{u},s)
:=
\Qtot(\boldsymbol{\tau},\mathbf{u},s;\theta_k,\psi_k),
\]
and consider the asynchronous TD update
\[
Q_{k+1}(x,u)
=
Q_k(x,u)
+
\alpha_k(x,u)
\Bigl[
    r_k
    +
    \gamma\max_{u'}Q_k(x'_k,u')
    -
    Q_k(x,u)
\Bigr],
\]
where $x$ denotes the joint state information used by QMIX during centralized training.
Then:
\begin{enumerate}
    \item $Q_k \to Q^*_{\mathrm{tot}}$ almost surely in $\ell_\infty$.
    \item The Bellman residual converges to zero almost surely:
    \[
        \|\mathcal{T}^*Q_k - Q_k\|_\infty \to 0,
    \]
    where $(\mathcal{T}^*Q)(x,u)=\mathbb{E}[r+\gamma\max_{u'}Q(x',u')\mid x,u]$.
    \item For the one-step target $y_k=r_k+\gamma\max_{u'}Q_k(x'_k,u')$, the expected squared TD error satisfies
    \[
    \mathbb{E}\big[(y_k-Q_k(x_k,u_k))^2\big] \to 0.
    \]
\end{enumerate}
\end{theorem}

\begin{proof}
\textit{Step 1 (Contraction).}
For any two joint Q-functions $Q$ and $Q'$,
\begin{align*}
\|(\mathcal{T}^* Q) - (\mathcal{T}^* Q')\|_\infty
&= \gamma \, \|\max_{\mathbf{u}'} Q(\cdot,\mathbf{u}',\cdot) - \max_{\mathbf{u}'} Q'(\cdot,\mathbf{u}',\cdot)\|_\infty \\
&\le \gamma \, \|Q - Q'\|_\infty,
\end{align*}
where the inequality follows from $|\max_a f(a) - \max_a g(a)| \le \max_a |f(a) - g(a)|$.
Since $\gamma < 1$, $\mathcal{T}^*$ is a $\gamma$-contraction, so it has a unique fixed point $Q^*_{\mathrm{tot}}$.

\textit{Step 2 (Almost-sure convergence of iterates).}
Under (A1)--(A4), the update is the standard asynchronous Q-learning recursion with bounded noise.
Classical stochastic-approximation results imply
\[
\|Q_k-Q^*_{\mathrm{tot}}\|_\infty \to 0 \quad \text{a.s.}
\]
\citep{sutton2018rl}.
Assumption (A5) ensures this fixed point is realizable by the QMIX monotonic parameterization.

\textit{Step 3 (Residual and TD-error convergence).}
Because $Q_k \to Q^*_{\mathrm{tot}}$ and $Q^*_{\mathrm{tot}}=\mathcal{T}^*Q^*_{\mathrm{tot}}$, continuity of $\mathcal{T}^*$ gives
\[
\|\mathcal{T}^*Q_k-Q_k\|_\infty \to 0 \quad \text{a.s.}
\]
The sampled TD error $\delta_k:=y_k-Q_k(x_k,u_k)$ is an unbiased noisy realization of this Bellman residual with bounded second moment, hence $\mathbb{E}[\delta_k^2]\to 0$.
\end{proof}

\subsection{Asymptotic Optimality of Decentralized Greedy Policy}

Finally, we characterize the greedy decentralized policy induced by the converged Q-function.

\begin{corollary}[Eventual Optimality Under an Action Gap]
\label{cor:policy}
Let $\pi_k(\boldsymbol{\tau},s)=\arg\max_{\mathbf{u}}\Qtot(\boldsymbol{\tau},\mathbf{u},s;\theta_k)$ be the greedy joint policy, and let its decentralized realization be $\pi_k^i(\tau^i)=\arg\max_{u^i}Q_i(\tau^i,u^i;\theta_k)$.
Assume the conditions of Theorem~\ref{thm:convergence} and define the optimal action gap
\[
\Delta_{\min}
=
\inf_{(\boldsymbol{\tau},s)}
\left[
Q^*_{\mathrm{tot}}(\boldsymbol{\tau},\mathbf{u}^*(\boldsymbol{\tau},s),s)
-
\max_{\mathbf{u}\neq \mathbf{u}^*(\boldsymbol{\tau},s)}Q^*_{\mathrm{tot}}(\boldsymbol{\tau},\mathbf{u},s)
\right]
>0.
\]
Then there exists $K<\infty$ such that for all $k\ge K$,
\[
\pi_k(\boldsymbol{\tau},s)=\mathbf{u}^*(\boldsymbol{\tau},s)
\quad \forall (\boldsymbol{\tau},s),
\]
and therefore the decentralized policy induced by $\{\pi_k^i\}$ is optimal: $J(\pi_k)=J(\pi^*)$ for all $k\ge K$.
\end{corollary}

\begin{proof}
From Theorem~\ref{thm:convergence}, $\|Q_k-Q^*_{\mathrm{tot}}\|_\infty \to 0$.
Hence there exists $K$ such that $\|Q_k-Q^*_{\mathrm{tot}}\|_\infty < \Delta_{\min}/2$ for all $k\ge K$.
This implies the maximizer of $Q_k$ equals the maximizer of $Q^*_{\mathrm{tot}}$ at every $(\boldsymbol{\tau},s)$, so $\pi_k=\pi^*$ for all $k\ge K$.
By Theorem~\ref{thm:igm}, this joint greedy policy is exactly realized by decentralized per-agent argmaxes.
\end{proof}

\paragraph{Remark (Finite-sample and function approximation).}
Theorems~\ref{thm:igm}--\ref{thm:convergence} and Corollary~\ref{cor:policy} are exact in the realizable tabular setting.
With deep function approximation (GNN, GRU, MLP), one typically obtains approximate fixed points rather than exact global guarantees.
In practice, experience replay, target networks, and gradient clipping reduce variance and stabilize training, while the IGM property still provides a principled decentralized execution rule.

\section{Notation Reference}
\label{app:notation}

Table~\ref{tab:notation} provides a summary of the main symbols used throughout this paper.

\begin{table}[h]
\centering
\caption{Summary of notation.}
\label{tab:notation}
\small
\begin{tabular}{cl}
\toprule
\textbf{Symbol} & \textbf{Description} \\
\midrule
$N$ & Number of agents \\
$\mathcal{N}$ & Set of agents $\{1,\dots,N\}$ \\
$\mathcal{S}$ & Global state space \\
$\mathcal{A}_i$ & Action space of agent $i$ (6 communication actions) \\
$\mathbf{u}_t$ & Joint action vector $(u_t^1,\dots,u_t^N)$ at round $t$ \\
$P$ & State transition kernel \\
$R$ & Team reward function \\
$\gamma$ & Discount factor \\
$\mathcal{G}_t = (\mathcal{N}, \mathcal{E}_t)$ & Communication graph at round $t$ \\
$A_t$ & Adjacency matrix of $\mathcal{G}_t$ \\
$\phi$ & Action-to-adjacency mapping (Section~\ref{sec:action_space}) \\
$x_t^i$ & Observation of agent $i$ at round $t$ \\
$\tau_t^i$ & Action-observation history of agent $i$ up to round $t$ \\
$h_v^{(l)}$ & GNN node embedding of agent $v$ at layer $l$ \\
$L$ & Number of GNN layers \\
$z_t^i$ & GRU hidden state of agent $i$ at round $t$ \\
$Q_i(\tau_t^i, u^i)$ & Individual action-value function of agent $i$ \\
$\Qtot$ & Joint action-value function (QMIX output) \\
$\Ebar_t$ & Average team reward at time $t$ \\
$s_t$ & Global state provided to the mixing network \\
$W_k, b_k$ & Mixing network weights and biases (from hypernetworks) \\
$\theta$ & Agent network parameters $\{\theta_{\mathrm{gnn}}, \theta_{\mathrm{rnn}}, \theta_{\mathrm{mlp}}\}$ \\
$\psi$ & Mixing network parameters \\
$\theta^-,\psi^-$ & Target network parameters \\
$\epsilon$ & Exploration rate for $\epsilon$-greedy \\
$\mathcal{B}$ & Replay buffer \\
$T$ & Number of communication rounds per episode \\
$w_{\mathrm{acc}}, w_{\mathrm{tok}}$ & Accuracy and token penalty reward weights \\
$\mathrm{max\_tokens}$ & Token normalization budget (10{,}000) \\
\bottomrule
\end{tabular}
\end{table}

\section{Qualitative Analysis}
\label{app:qualitative}

To complement the quantitative results in Section~\ref{sec:main_results}, we present representative examples from each domain to illustrate \emph{why} Agent Q-Mix's learned topologies are effective.

\paragraph{Mathematics (AIME).}
On an AIME combinatorics problem, Agent Q-Mix selects \texttt{broadcast\_all} in round~1, enabling all three agents (two MathSolvers and one AnalyzeAgent) to share initial reasoning paths. In round~2, the policy shifts to \texttt{selective\_query} for the MathSolvers and \texttt{aggregate\_refine} for the AnalyzeAgent, allowing targeted verification of the intermediate answer. The AnalyzeAgent identifies an off-by-one error in one solver's counting argument, and the FinalRefer node selects the corrected answer. This deliberative-then-convergent pattern is typical on hard math tasks: broad exploration followed by focused verification.

\paragraph{Coding (HumanEval).}
On a string-processing task where the base model already produces correct code, the trained policy selects \texttt{solo\_process} for all agents in both rounds, incurring minimal token overhead. The AnalyzeAgent confirms correctness without requesting additional communication. This demonstrates the policy's ability to \emph{suppress unnecessary communication} on easy problems, directly contributing to token efficiency (Table~\ref{tab:token}).

\paragraph{Reasoning (MMLU-Pro).}
On a multi-choice reasoning question about economic policy, the policy selects \texttt{debate\_check} between the two ReasoningAgents and \texttt{selective\_query} from the AnalyzeAgent to one ReasoningAgent. The adversarial exchange surfaces a subtle misinterpretation of the question stem, and the AnalyzeAgent's targeted query resolves the disagreement. The resulting answer flips from incorrect (under solo reasoning) to correct, illustrating how structured debate improves factual accuracy on ambiguous questions.

\section{Ablation Studies}
\label{app:ablation}

We conduct four ablation studies on key design choices of Agent Q-Mix using Gemini-3.1-Flash-Lite as the base LLM. Results are summarized in Figure~\ref{fig:ablation}.

\begin{figure}[h]
\centering
\includegraphics[width=\linewidth]{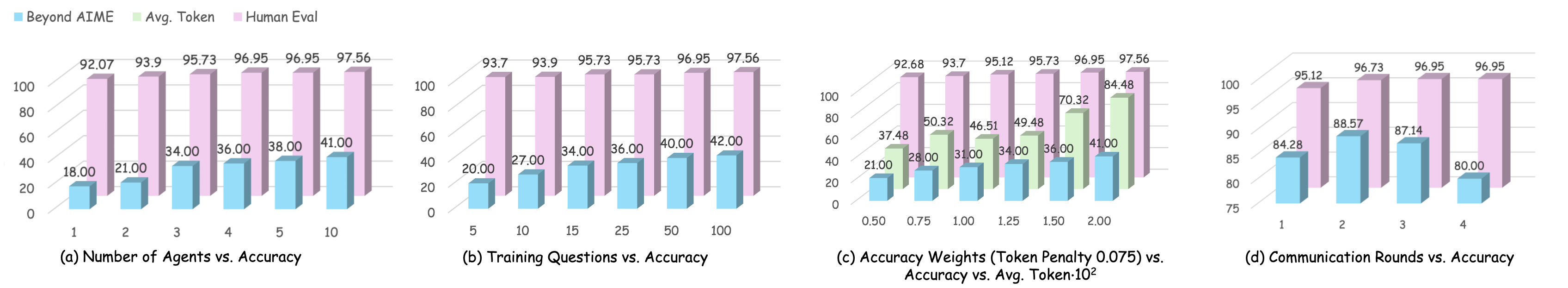}
\caption{Ablation studies on Gemini-3.1-Flash-Lite. (a)~Number of agents vs.\ accuracy on Beyond-AIME, average token cost, and HumanEval. (b)~Number of training examples vs.\ accuracy. (c)~Accuracy reward weight $w_{\mathrm{acc}}$ vs.\ accuracy and average token cost (with $w_{\mathrm{tok}}=0.075$). (d)~Number of communication rounds vs.\ accuracy.}
\label{fig:ablation}
\end{figure}

\paragraph{Number of agents (Figure~\ref{fig:ablation}a).}
We vary the team size from 1 to 10 agents. Beyond-AIME accuracy increases sharply from 18.00\% (1 agent) to 38.00\% (4 agents), then plateaus at 41.00\% with 10 agents. HumanEval accuracy remains near-perfect ($>$96\%) across all team sizes, confirming that the policy avoids unnecessary communication on easy tasks. Token usage scales sublinearly due to the learned topology suppressing redundant exchange.

\paragraph{Training data size (Figure~\ref{fig:ablation}b).}
We train with 5 to 100 examples per domain. With only 15 examples (our default), accuracy reaches 95.73\% on the aggregate metric, comparable to 50 or 100 examples (96.95--97.56\%). This validates the sample efficiency of the QMIX policy: a small number of diverse training tasks suffices to learn generalizable communication strategies. We select 15 examples per domain as the default to minimize training cost while maintaining strong generalization.

\paragraph{Reward weights (Figure~\ref{fig:ablation}c).}
We sweep the accuracy reward weight $w_{\mathrm{acc}}$ from 0.50 to 2.00 while fixing $w_{\mathrm{tok}}=0.075$. Accuracy increases monotonically with $w_{\mathrm{acc}}$ up to 1.50, while token usage grows more gradually. At $w_{\mathrm{acc}}=1.50$, the policy achieves 96.95\% accuracy with moderate token cost; beyond this point, diminishing returns suggest an accuracy-efficiency frontier. We select $w_{\mathrm{acc}}=1.50$ for Gemini-3.1-Flash-Lite and $w_{\mathrm{acc}}=1.25$ for GPT-OSS:120B based on this analysis.

\paragraph{Communication rounds (Figure~\ref{fig:ablation}d).}
We vary the number of rounds $T$ from 1 to 4. With $T=1$, accuracy is 84.28\%; it improves to 87.14\% at $T=2$ and 96.95\% at $T=3$, then slightly decreases to 96.95\% at $T=4$. This supports our default configuration of $T=3$ rounds for mathematics and $T=2$ for coding and reasoning, where fewer rounds suffice due to lower task complexity.

\section{Agent Prompts}
\label{app:prompts}

Below we provide the full system prompts used for each agent role. Each agent type receives a domain-specific system prompt that defines its expertise and output format. Agents are implemented as LLM wrappers: each agent constructs a prompt from (i)~the system prompt below, (ii)~the task description, (iii)~spatial information from predecessors in the current communication graph, and (iv)~temporal information from previous communication rounds.

\subsection{Coding Agents (CodeWriter)}

CodeWriter agents cycle through five specialized roles. Each role's system prompt is shown below.

\begin{tcolorbox}[colback=lightpurple, colframe=darkgray, title={\small\textbf{Project Manager}}]
\small You are a project manager. You will be given a function signature and its docstring by the user. You are responsible for overseeing the overall structure of the code, ensuring that the code is structured to complete the task. Implement code concisely and correctly without pursuing over-engineering. You need to suggest optimal design patterns to ensure that the code follows best practices for maintainability and flexibility. You can specify the overall design of the code, including the classes that need to be defined (maybe none) and the functions used (maybe only one function). I hope your reply will be more concise. Preferably within fifty words.
\end{tcolorbox}

\begin{tcolorbox}[colback=lightpurple, colframe=darkgray, title={\small\textbf{Algorithm Designer}}]
\small You are an algorithm designer. You will be given a function signature and its docstring by the user. You need to specify the specific design of the algorithm, including the classes that may be defined and the functions used. You need to generate the detailed documentation, including explanations of the algorithm, usage instructions, and API references. When the implementation logic is complex, you can give the pseudocode logic of the main algorithm.
\end{tcolorbox}

\begin{tcolorbox}[colback=lightpurple, colframe=darkgray, title={\small\textbf{Programming Expert}}]
\small You are a programming expert. You will be given a function signature and its docstring by the user. You may be able to get the output results of other agents. They may have passed internal tests, but they may not be completely correct. Write your full implementation (restate the function signature). Use a Python code block to write your response. Do not include anything other than Python code blocks in your response. Do not change function names and input variable types in tasks.
\end{tcolorbox}

\begin{tcolorbox}[colback=lightpurple, colframe=darkgray, title={\small\textbf{Test Analyst}}]
\small You are a test analyst. You will be given a function signature and its docstring by the user. You need to provide problems in the current code or solution based on the test data and possible test feedback. You need to provide additional special use cases, boundary conditions, etc. You can point out any potential errors in the code. Preferably within fifty words.
\end{tcolorbox}

\begin{tcolorbox}[colback=lightpurple, colframe=darkgray, title={\small\textbf{Bug Fixer}}]
\small You are a bug fixer. You will be given a function signature and its docstring by the user. You need to provide modified and improved python code based on the current overall code design, algorithm framework, code implementation or test problems. Write your full implementation (restate the function signature). Use a Python code block to write your response. Do not change function names and input variable types in tasks.
\end{tcolorbox}

\subsection{Mathematics Agents (MathSolver)}

MathSolver agents cycle through four roles. Each also receives few-shot exemplars (omitted here for space; see the codebase).

\begin{tcolorbox}[colback=lightpurple, colframe=darkgray, title={\small\textbf{Math Solver}}]
\small You are a math expert. You will be given a math problem and hints from other agents. Give your own solving process step by step based on hints. The last line of your output contains only the final result without any units, for example: The answer is 140.
\end{tcolorbox}

\begin{tcolorbox}[colback=lightpurple, colframe=darkgray, title={\small\textbf{Mathematical Analyst}}]
\small You are a mathematical analyst. You will be given a math problem, analysis and code from other agents. You need to first analyze the problem-solving process step by step, where the variables are represented by letters. Then you substitute the values into the analysis process to perform calculations and get the results. The last line of your output contains only the final result.
\end{tcolorbox}

\begin{tcolorbox}[colback=lightpurple, colframe=darkgray, title={\small\textbf{Programming Expert (Math)}}]
\small You are a programming expert. Integrate step-by-step reasoning and Python code to solve math problems. Analyze the question and write functions to solve the problem. The function should not take any arguments and use the final result as the return value. Use a Python code block to write your response.
\end{tcolorbox}

\begin{tcolorbox}[colback=lightpurple, colframe=darkgray, title={\small\textbf{Inspector}}]
\small You are an Inspector. Check whether the logic/calculation of the problem solving and analysis process is correct (if present). Check whether the code corresponds to the solution analysis (if present). Give your own solving process step by step based on hints. The last line of your output contains only the final result.
\end{tcolorbox}

\subsection{Reasoning Agents (ReasoningAgent)}

\begin{tcolorbox}[colback=lightpurple, colframe=darkgray, title={\small\textbf{ReasoningAgent System Prompt}}]
\small You are a world-class reasoning expert. Think step by step, consider multiple perspectives, and provide a well-reasoned answer. Be concise but thorough.
\end{tcolorbox}

For MMLU-Pro and HLE tasks, agents additionally receive MCQ-specific formatting constraints instructing them to respond with a single letter (A--J for MMLU-Pro, A--V for HLE) on the first line.

\subsection{AnalyzeAgent}

\begin{tcolorbox}[colback=lightpurple, colframe=darkgray, title={\small\textbf{AnalyzeAgent System Prompt}}]
\small You are an expert critical analyzer. Examine the problem carefully, identify potential errors in other agents' solutions, and provide a thorough analysis. Be precise and constructive.
\end{tcolorbox}

\subsection{FinalRefer (Decision Node)}

The FinalRefer decision node receives all agents' outputs and is prompted with a domain-specific role and constraint:

\begin{tcolorbox}[colback=lightpurple, colframe=darkgray, title={\small\textbf{FinalRefer --- Coding}}]
\small \textbf{Role:} You are the top decision-maker and are good at analyzing and summarizing other people's opinions, finding errors and giving final answers. And you are an AI that only responds with only python code.\\
\textbf{Constraint:} Write your full implementation. If the prompt contains code that passed internal testing, choose the most reliable reply. Use a Python code block to write your response.
\end{tcolorbox}

\begin{tcolorbox}[colback=lightpurple, colframe=darkgray, title={\small\textbf{FinalRefer --- Mathematics}}]
\small \textbf{Role:} You are the top decision-maker. Good at analyzing mathematical problems, judging and summarizing other people's solutions critically, and giving final answers to math problems.\\
\textbf{Constraint:} Find the most reliable answer based on the analysis and results of other agents. The last line of your output contains only the final result, e.g.: The answer is 140.
\end{tcolorbox}

\begin{tcolorbox}[colback=lightpurple, colframe=darkgray, title={\small\textbf{FinalRefer --- Reasoning (MMLU-Pro / HLE)}}]
\small \textbf{Role:} You are the top decision-maker and are good at solving problems and analyzing other people's opinions, thinking critically and giving final answers of the task.\\
\textbf{Constraint:} Choose the correct answer. Your reply must only contain one letter and cannot have any other characters.
\end{tcolorbox}

\end{document}